\documentclass[11pt]{article}
\usepackage{pifont,fullpage,xspace,epsfig, wrapfig,times}
\usepackage[small,it]{caption}
\usepackage{amsmath, amsthm, amssymb,times,paralist}
\usepackage{latexsym}
\usepackage{subfigure}
\usepackage[numbers,longnamesfirst]{natbib}
\usepackage{hyperref}
\usepackage{algorithm,algorithmic}

\usepackage{graphicx}
\newtheorem{theorem}{Theorem}[section]
\newtheorem{lemma}[theorem]{Lemma}
\newtheorem{definition}{Definition}
\newtheorem{claim}{Claim}
\newtheorem{proposition}[theorem]{Proposition}
\newenvironment{remark}{\smallskip\noindent {\emph{\bf Remark:}}}{\smallskip}

\makeatletter
\newcommand\footnoteref[1]{\protected@xdef\@thefnmark{\ref{#1}}\@footnotemark}
\makeatother
\def \y {\mathbf y}

\def \e {\mathbf e}
\def \x {\mathbf x}

\def \u {\mathbf u}
\def \z {\mathbf z}
\def \v {\mathbf v}

\def \R {\mathbb{R}}
\def \N {\mathbb{N}}

\def \XXX {\mathcal{X}}
\def \HHH {\mathcal{H}}
\def \AAA {\mathcal{A}}
\newenvironment{CompactEnumerate}{
\begin{list}{\alph{enumi})}{%
\usecounter{enumi}
\setlength{\leftmargin}{12pt}
\setlength{\itemindent}{3pt}
\setlength{\topsep}{3pt}
\setlength{\itemsep}{1pt}
}}
{\end{list}}
\newenvironment{CompactEn}{
\begin{list}{\arabic{enumi})}{%
\usecounter{enumi}
\setlength{\leftmargin}{12pt}
\setlength{\itemindent}{3pt}
\setlength{\topsep}{3pt}
\setlength{\itemsep}{1pt}
}}
{\end{list}}

\def \N {\mathbb{N}}
\def \R {\mathbb{R}}

\def \E {\mathbb{E}}

\def \etc {,\ldots,}

\newcommand{\norm}[1]{\left \| #1 \right \|}

\begin{document}
\title{\Large Spectral Norm of Random Kernel Matrices with Applications to Privacy}
\author{Shiva Kasiviswanathan\thanks{Samsung Research America, \texttt{kasivisw@gmail.com}. Part of the work done while the author was at General Electric Research.} \and \ \ \ \ Mark Rudelson\thanks{University of Michigan, \texttt{rudelson@umich.edu}.}  
}
\date{}
\maketitle
\begin{abstract}

Kernel methods are an extremely popular set of techniques used for many important machine learning and data analysis applications. In addition to having good practical performance, these methods are supported by a well-developed theory. Kernel methods use an implicit mapping of the input data into a high dimensional feature space defined by a kernel function, i.e., a function returning the inner product between the images of two data points in the feature space. Central to any kernel method is the kernel matrix, which is built by evaluating the kernel function on a given sample dataset. 

In this paper, we initiate the study of non-asymptotic spectral theory of random kernel matrices. These are $n \times n$ random matrices whose $(i,j)$th entry is obtained by evaluating the kernel function on $\x_i$ and $\x_j$, where $\x_1,\dots,\x_n$ are a set of $n$ independent random high-dimensional vectors. Our main contribution is to obtain tight upper bounds on the spectral norm (largest eigenvalue) of random kernel matrices constructed by commonly used kernel functions based on polynomials and Gaussian radial basis. 

As an application of these results, we provide lower bounds on the distortion needed for releasing the coefficients of kernel ridge regression under attribute privacy, a general privacy notion which captures a large class of privacy definitions. Kernel ridge regression is standard method for performing non-parametric regression that regularly outperforms traditional regression approaches in various domains. Our privacy distortion lower bounds are the first for any kernel technique, and our analysis assumes realistic scenarios for the input, unlike all previous lower bounds for other release problems which only hold under very restrictive input settings.   
\end{abstract}

\section{Introduction}
In recent years there has been significant progress in the development and application of kernel methods for many practical machine learning and data analysis problems. Kernel methods are regularly used for a range of problems such as classification (binary/multiclass), regression, ranking, and unsupervised learning, where they are known to almost always outperform ``traditional'' statistical techniques~\cite{scholkopf2001learning,shawe2004kernel}. At the heart of kernel methods is the notion of {\em kernel function}, which is a real-valued function of two variables. The power of kernel methods stems from the fact for every (positive definite) kernel function it is possible to define an inner-product and a lifting (which could be nonlinear) such that inner-product between any two lifted datapoints can be quickly computed using the kernel function evaluated at those two datapoints.  This allows for introduction of nonlinearity into the traditional optimization problems (such as Ridge Regression, Support Vector Machines, Principal Component Analysis) without unduly complicating them. 

The main ingredient of any kernel method is the {\em kernel matrix}, which is built using the kernel function, evaluated at given sample points. Formally, given a kernel function $\kappa: \XXX \times \XXX \rightarrow \R$ and a sample set $\x_1,\dots,\x_n$, the kernel matrix $K$ is an $n \times n$ matrix with its $(i,j)$th entry $K_{ij}=\kappa(\x_i,\x_j)$. Common choices of kernel functions include the polynomial kernel ($\kappa(\x_i,\x_j) = (a \langle \x_i,\x_j \rangle +b)^p$, for $p \in \N$) and the Gaussian kernel  ($\kappa(\x_i,\x_j) = \exp(-a  \| \x_i - \x_j  \|^2)$, for $a > 0$)~\cite{scholkopf2001learning,shawe2004kernel}. 

In this paper, we initiate the study of non-asymptotic spectral properties of {\em random kernel matrices}. A random kernel matrix, for a kernel function $\kappa$, is the kernel matrix $K$ formed by $n$ independent random vectors $\x_1,\dots,\x_n \in \R^d$. The prior work on random kernel matrices~\cite{karoui2010spectrum,cheng2013spectrum,do2013spectrum} have established various interesting properties of the spectral distributions of these matrices in the asymptotic sense (as $n, d \rightarrow \infty$). However, analyzing algorithms based on kernel methods typically requires understanding of the spectral properties of these random kernel matrices for {\em large, but fixed} $n, d$. A similar parallel also holds in the study of the spectral properties of ``traditional'' random matrices, where recent developments in the non-asymptotic theory of random matrices have complemented the classical random matrix theory that was mostly focused on asymptotic spectral properties~\cite{V11,rudelson2014recent}.

We investigate upper bounds on the largest eigenvalue (spectral norm) of random kernel matrices for polynomial and Gaussian kernels. We show that for inputs $\x_1,\dots,\x_n$ drawn independently from a wide class of probability distributions over $\R^d$ (satisfying the subgaussian property), the spectral norm of a random kernel matrix constructed using a polynomial kernel of degree $p$, with high probability, is roughly bounded by $O(d^p n)$. In a similar setting, we show that the spectral norm of a random kernel matrix constructed using a Gaussian kernel is bounded by $O(n)$, and with high probability,  this bound reduces to $O(1)$ under some stronger assumptions on the subgaussian distributions. These bounds are almost tight. Since the entries of a random kernel matrix are highly correlated, the existing techniques prevalent in random matrix theory can not be directly applied. We overcome this problem by careful splitting and conditioning arguments on the random kernel matrix. Combining these with subgaussian norm concentrations form the basis of our proofs.

\paragraph{Applications.} Largest eigenvalue of kernel matrices plays an important role in the analysis of many machine learning algorithms. Some examples include, bounding the Rademacher complexity for multiple kernel learning~\cite{lanckriet2004learning}, analyzing the convergence rate of conjugate gradient technique for matrix-valued kernel learning~\cite{sindhwani2012scalable}, and establishing the concentration bounds for eigenvalues of kernel matrices~\cite{jia2009accurate,shawe2005eigenspectrum}. 

In this paper, we focus on an application of these eigenvalue bounds to an important problem arising while analyzing sensitive data. Consider a curator who manages a database of sensitive information but wants to release statistics about how a {\em sensitive} attribute (say, disease) in the database relates with some {\em nonsensitive} attributes (e.g., postal code, age, gender, etc). This setting is widely considered in the applied data privacy literature, partly since it arises with medical and retail data.  Ridge regression is a well-known approach for solving these problems due to its good generalization performance. Kernel ridge regression is a powerful technique for building nonlinear regression models that operate by combining ridge regression with kernel methods~\cite{saunders1998ridge}.\!\footnote{We provide a brief coverage of the basics of kernel ridge regression in Section~\ref{sec:private}.} We present a {\em linear reconstruction attack}\footnote{In a linear reconstruction attack, given the released information $\rho$, the attacker constructs a system of approximate linear equalities of the form $A\z \approx \rho$ for a matrix $A$ and attempts to solve for $\z$.} that reconstructs, with high probability, almost all the sensitive attribute entries given sufficiently accurate approximation of the kernel ridge regression coefficients.  We consider reconstruction attacks against {\em attribute privacy}, a loose notion of privacy, where the goal is to just avoid any gross violation of privacy. Concretely, the input is assumed to be a database whose $i$th row (record for individual $i$) is $(\x_i,y_i)$ where $\x_i \in \R^d$ is assumed to be known to the attacker (public information) and $y_i \in \{0,1\}$ is the sensitive attribute, and a privacy mechanism is {\em attribute non-private} if the attacker can consistently reconstruct a large fraction of the sensitive attribute ($y_1,\dots,y_n$). We show that any privacy mechanism that always adds $\approx o(1/(d^p n))$ noise\footnote{\label{note1}Ignoring the dependence on other parameters, including the regularization parameter of ridge regression.} to each coefficient of a polynomial kernel ridge regression model is attribute non-private. Similarly any privacy mechanism that always adds $\approx o(1)$ noise\footnoteref{note1} to each coefficient of a Gaussian kernel ridge regression model is attribute non-private. As we later discuss, there exists natural settings of inputs under which these kernel ridge regression coefficients, even without the privacy constraint, have the same magnitude as these noise bounds, implying that privacy comes at a steep price. While the linear reconstruction attacks employed in this paper themselves are well-known~\cite{DY08,KRSU10,KRS13}, these are the first attribute privacy lower bounds that: \renewcommand{\labelenumi}{(\roman{enumi})}\begin{inparaenum}  \item are applicable to any kernel method and \item work for any $d$-dimensional data, analyses of all previous attacks (for other release problems) require $d$ to be comparable to $n$. \end{inparaenum} Additionally, unlike previous reconstruction attack analyses, our bounds hold for a wide class of realistic distributional assumptions on the data. 

\subsection{Comparison with Related Work} \label{sec:related}
In this paper, we study the largest eigenvalue of an $n \times n$ random kernel matrix in the non-asymptotic sense. The general goal with studying non-asymptotic theory of random matrices is to understand the spectral properties of random matrices, which are valid with high probability for matrices of a large fixed size. This is contrast with the existing theory on random kernel matrices which have focused on the asymptotics of various spectral characteristics of these random matrices, when the dimensions of the matrices tend to infinity. Let $\x_1,\dots,\x_n \in \R^d$ be $n$ i.i.d. random vectors. For any $F : \R^d \times \R^d \times \R \rightarrow \R$, symmetric in the first two variables, consider the random kernel matrix $K$ with $(i,j)$th entry $K_{ij} = F(\x_i,\x_j,d)$. El Karoui~\cite{karoui2010spectrum} considered the case where $K$ is generated by either the {\em inner-product kernels} (i.e., $F(\x_i,\x_j,d) = f(\langle \x_i,\x_j \rangle,d)$) or the {\em distance kernels} (i.e., $F(\x_i,\x_j,d) = f(\| \x_i - \x_j \|^2,d)$). It was shown there that under some assumptions on $f$ and on the distributions of $\x_i$'s, and in the ``large $d$, large $n$'' limit (i.e., and $d, n \rightarrow \infty$ and $d/n \rightarrow (0,\infty)$): a)  the non-linear kernel matrix converges asymptotically in spectral norm to a linear kernel matrix, and b) there is a weak convergence of the limiting spectral density. These results were recently strengthened in different directions by Cheng~\emph{et al.}\ \cite{cheng2013spectrum} and Do~\emph{et al.}\ \cite{do2013spectrum}. To the best of our knowledge, ours is the first paper investigating the non-asymptotic spectral properties of a random kernel matrix. 

Like the development of non-asymptotic theory of traditional random matrices has found multitude of applications in areas including statistics, geometric functional analysis, and compressed sensing~\cite{V11}, we believe that the growth of a non-asymptotic theory of random kernel matrices will help in better understanding of many machine learning applications that utilize kernel techniques. 

The goal of \emph{private data analysis} is to release global, statistical properties of a database while protecting the privacy of the individuals whose information the database contains. Differential privacy~\cite{DMNS06} is a formal notion of privacy tailored to private data analysis. Differential privacy requires, roughly, that any single individual's data have little effect on the outcome of the analysis. A lot of recent research has gone in developing differentially private algorithms for various applications, including kernel methods~\cite{jain2013differentially}. A typical objective here is to release as accurate an approximation as possible to some function $f$ evaluated on a database $D$.

In this paper, we follow a complementary line of work that seeks to understand how much distortion (noise) is necessary to privately release some particular function $f$ evaluated on a database containing sensitive information~\cite{DiNi03,DMT07,DY08,KRSU10,De11,merener2011polynomial,Choromanski:2012:PDA:2213556.2213570,MuthuN12,KRS13}. The general idea here, is to provide {\em reconstruction attacks}, which are attacks that can reconstruct (almost all of) the sensitive part of database $D$ given sufficiently accurate approximations to $f(D)$. Reconstruction attacks violate any {\em reasonable} notion of privacy (including, differential privacy), and the existence of these attacks directly translate into lower bounds on distortion needed for privacy.

Linear reconstruction attacks were first considered in the context of data privacy by Dinur and Nissim~\cite{DiNi03}, who showed that any mechanism which answers $\approx n \log n$ random inner product queries on a database in $\{0,1\}^n$ with $o(\sqrt{n})$ noise per query is not private. Their attack was subsequently extended in various directions by~\cite{DMT07,DY08,merener2011polynomial,Choromanski:2012:PDA:2213556.2213570}.

The results that are closest to our work are the attribute privacy lower bounds analyzed for releasing $k$-way marginals~\cite{KRSU10,De11}, linear/logistic regression parameters~\cite{KRS13}, and a subclass of statistical $M$-estimators~\cite{KRS13}. Kasiviswanathan~\emph{et al.}\ \cite{KRSU10} showed that,  if $d = \tilde{\Omega}(n^{1/(k-1)})$, then any mechanism which releases all $k$-way marginal tables with $o(\sqrt{n})$ noise per entry is attribute non-private.\!\footnote{The $\tilde{\Omega}$ notation hides polylogarithmic factors.} These noise bounds were improved by De~\cite{De11}, who presented an attack that can tolerate a constant fraction of entries with arbitrarily high noise, as long as the remaining entries have $o(\sqrt{n})$ noise. Kasiviswanathan~\emph{et al.}\ \cite{KRS13} recently showed that, if $d=\Omega(n)$, then any mechanism which releases $d$ different linear or logistic regression estimators each with $o(1/\sqrt{n})$ noise is attribute non-private. They also showed that this lower bound extends to a subclass of statistical $M$-estimator release problems. 
%All these previous results, utilize linear reconstruction attacks (first considered in the context of data privacy by Dinur and Nissim~\cite{DiNi03}) to obtain these attribute privacy lower bounds. 
A point to observe is that in all the above referenced results, $d$ has to be comparable to $n$, and this dependency looks unavoidable in those results due to their use of least singular value bounds. However, in this paper, our privacy lower bounds hold for all values of $d, n$ ($d$ could be $\ll n$). Additionally, all the previous reconstruction attack analyses critically require the $\x_i$'s to be drawn from product of univariate subgaussian distributions, whereas our analysis here holds for any $d$-dimensional subgaussian distributions (not necessarily product distributions), thereby is more widely applicable. The subgaussian assumption on the input data is quite common in the analysis of machine learning algorithms~\cite{bousquet2004advanced}.

\section{Preliminaries}
\paragraph{Notation.} We use $[n]$ to denote the set $\{1 \etc n\}$. $d_H(\cdot,\cdot)$ measures the Hamming distance. Vectors used in the paper are by default column vectors and are denoted by boldface letters. For a vector $\v$, $\v^\top$ denotes its transpose and $\|\v\|$ denotes its Euclidean norm. For two vectors $\v_1$ and $\v_2$, $\langle \v_1, \v_2 \rangle$ denotes the inner product of $\v_1$ and $\v_2$. For a matrix $M$, $\|M\|$ denotes its spectral norm, $\|M\|_F$ denotes its Frobenius norm, and $M_{ij}$ denotes its $(i,j)$th entry. $\mathbb{I}_n$ represents the identity matrix in dimension $n$. 
The unit sphere in $d$ dimensions centered at origin is denoted by $S^{d-1} = \{\z \,:\, \|\z\| =1, \z \in \R^d \}$. Throughout this paper $C,c,C'$, also with subscripts, denote absolute constants (i.e., independent of $d$ and $n$), whose value may change from line to line.

\subsection{Background on Kernel Methods}
We provide a very brief introduction to the theory of kernel methods; see the many books on the topic~\cite{scholkopf2001learning,shawe2004kernel} for further details.
\begin{definition} [Kernel Function] \label{defn:mercer}
Let $\XXX$ be a non-empty set. Then a function $\kappa: \XXX \times \XXX \rightarrow \R$ is called a kernel function on $\XXX$ if there exists a Hilbert space $\HHH$ over $\R$ and a map $\phi: \XXX \rightarrow \HHH$ such that for all $\x,\y \in \XXX$, we have
$$\kappa(\x,\y) = \langle \phi(\x),\phi(\y) \rangle_{\HHH}.$$
\end{definition}

For any {\em symmetric and positive semidefinite}\footnote{A positive definite kernel is a function $\kappa : \XXX \times \XXX \rightarrow \R$ such that for any $n \geq 1$, for any finite set of points $\{\x_i\}_{i=1}^n$ in $\XXX$ and real numbers $\{a_i\}_{i=1}^n$, we have $\sum_{i,j =1}^n a_i a_j \kappa(\x_i,\x_j) \geq 0$.}  kernel $\kappa$, by Mercer's theorem~\cite{mercer1909functions} there exists: \begin{inparaenum} \item a unique functional Hilbert space $\HHH$ (referred to as the reproducing kernel Hilbert space, Definition~\ref{defn:RKHS}) on $\XXX$ such that  $\kappa(\cdot,\cdot)$ is the inner product in the space and \item a map $\phi$ defined as $\phi(\x) := \kappa(\cdot,\x)$\footnote{$\kappa(\cdot,\x)$ is a vector with entries $\kappa(\x',\x)$ for all $\x' \in  \XXX$.} that satisfies Definition~\ref{defn:mercer}\end{inparaenum}. The function $\phi$ is called the {\em feature map} and the space $\HHH$ is called the {\em feature space}.

\begin{definition} [Reproducing Kernel Hilbert Space] \label{defn:RKHS}
A kernel $\kappa(\cdot,\cdot)$ is a reproducing kernel of a Hilbert space $\HHH$ if $\forall f \in \HHH$, $f(\x) = \langle \kappa(\cdot,\x), f(\cdot) \rangle_\HHH$. For a (compact) $\XXX \subseteq \R^d$, and a Hilbert space $\HHH$ of functions $f : \XXX \rightarrow \R$, we say $\HHH$ is a Reproducing Kernel Hilbert Space if there $\exists \kappa : \XXX \times \XXX \rightarrow \R$, s.t.: a) $\kappa$ has the reproducing property, and b) $\kappa$ spans $\HHH = \overline{\mbox{span}\{\kappa(\cdot, \x) : \x \in \XXX\}}$.
\end{definition}

A standard idea used in the machine-learning community (commonly referred to as the ``kernel trick'') is that kernels allow for the computation of inner-products in high-dimensional feature spaces ($\langle \phi(\x),\phi(\y) \rangle_{\HHH}$) using simple functions defined on pairs of input patterns ($\kappa(\x,\y)$), without knowing the $\phi$ mapping explicitly. This trick allows one to efficiently solve a variety of non-linear optimization problems. Note that there is no restriction on the dimension of the feature maps ($\phi(\x)$), i.e., it could be of infinite dimension.

Polynomial and Gaussian are two popular kernel functions that are used in many machine learning and data mining tasks such as classification, regression, ranking, and structured prediction. Let the input space $\XXX =\R^d$. For $\x, \y \in \R^d$, these kernels are defined as:

\newcommand{\resref}[1]{{\bf (\ref{res:#1})}}
\begin{list}{{\bf (\arabic{enumi})}}{\usecounter{enumi}
\setlength{\leftmargin}{\parindent}
\setlength{\listparindent}{\parindent}
\setlength{\parsep}{0pt}}
\item \label{res:poly} {\bf Polynomial Kernel}: $\kappa(\x,\y) = (a \langle \x,\y \rangle + b)^p$, with parameters $a,b \in \R$ and $p \in \N$. Here $a$ is referred to as the slope parameter, $b \geq 0$ trades off the influence of higher-order versus lower-order terms in the polynomial, and $p$ is the polynomial degree. For an input $\x \in \R^d$, the feature map $\phi(\x)$ of the polynomial kernel is a vector with a polynomial in $d$ number of dimensions~\cite{scholkopf2001learning}.
\item \label{res:rbf} {\bf Gaussian Kernel} (also frequently referred to as the {\em radial basis kernel}):  $\kappa(\x,\y) = \exp \left ( - a \| \x -\y \|^2 \right )$ with real parameter $a > 0$. The value of $a$ controls the locality of the kernel with low values indicating that the influence  of a single point is ``far'' and vice-versa~\cite{scholkopf2001learning}. An equivalent popular formulation, is to set $a = 1/2\sigma^2$, and hence, $\kappa(\x,\y) = \exp \left ( - \| \x -\y \|^2/2\sigma^2 \right )$.   For an input $\x \in \R^d$, the feature map $\phi(\x)$ of the Gaussian kernel is a vector of infinite dimensions~\cite{scholkopf2001learning}. Note that while we focus on the Gaussian kernel in this paper, the extension of our results to other exponential kernels such as the {\em Laplacian kernel} (where $\kappa(\x,\y) = \exp \left ( - a \| \x -\y \|_1 \right )$), is quite straightforward.
\end{list}

\subsection{Background on Subgaussian Random Variables} 
Let us start by formally defining subgaussian random variables and vectors.
\begin{definition} [Subgaussian Random Variable and Vector]\label{def:subgauss} 
We call a random variable $x \in \R$ subgaussian if there exists  a constant $C > 0$ if $\Pr[ |x| > t]  \leq 2 \exp(-t^2/C^2)$ for all $t > 0$. We say that a random vector $\x \in \R^d$ is subgaussian if the one-dimensional marginals $\langle \x,\y \rangle$ are subgaussian random variables for all $\y \in \R^d$.
\end{definition}
The class of subgaussian random variables includes many random variables that arise naturally in data analysis, such as standard normal, Bernoulli, spherical, bounded (where the random variable $x$ satisfies $| x | \leq M$ {\em almost surely} for some fixed $M$). The natural generalization of these random variables to higher dimension are all subgaussian random vectors. For many {\em isotropic convex sets}\footnote{A convex set $\mathcal{K}$  in $\R^d$ is called isotropic if a random vector chosen uniformly from $\mathcal{K}$ according to the volume is isotropic. A random vector $\x \in \R^d$ is isotropic if for all $\y \in \R^d$, $\E [ \langle \x,\y \rangle^2] = \| \y \|^2.$}
$\mathcal{K}$ (such as the hypercube), a random vector $\x$ uniformly distributed in $\mathcal{K}$ is subgaussian.

\begin{definition}[Norm of Subgaussian Random Variable and Vector]
The $\psi_2$-norm of a subgaussian random variable $x \in \R$, denoted by $\| x \|_{\psi_2}$ is:
$$ \| x \|_{\psi_2} = \inf \left \{ t > 0 \,:\, \E[\exp(|x|^2/t^2)] \leq 2 \right \}. $$
The $\psi_2$-norm of a subgaussian random vector $\x \in \R^d$ is:
$$ \| \x \|_{\psi_2} = \sup_{\y \in S^{d-1}} \; \| \langle \x,\y \rangle \|_{\psi_2}.$$
\end{definition}

\begin{claim}[Vershynin~\cite{V11}] \label{claim:norm}
Let $\x \in \R^d$ be a subgaussian random vector. Then there exists a constant $C > 0$, such that $\Pr [ | \x | > t]  \leq 2\exp(-C t^2/ \| \x \|^2_{\psi_2})$.
\end{claim}

Consider a subset $T$ of $\R^d$, and let $\epsilon > 0$. An $\epsilon$-net of $T$ is a subset $\mathcal{N} \subseteq T$ such that for every $\x \in T$, there exists a $\z \in \mathcal{N}$ such that $\|\x - \z \|   \leq \epsilon$. We would use the following well-known result about the size of $\epsilon$-nets.

\begin{proposition}[Bounding the size of an $\epsilon$-Net~\cite{V11}]~\label{prop:nets}
Let $T$ be a subset of $S^{d-1}$ and let $\epsilon > 0$. Then there exists an $\epsilon$-net of $T$ of cardinality at most $(1+2/\epsilon)^d$.
\end{proposition}
The proof of the following claim follows by standard techniques.
\begin{claim} [~\cite{V11}]  \label{claim:net}
Let $\mathcal{N}$ be a $1/2$-net of $S^{d-1}$. Then for any $\x \in \R^d$, $\| \x \| \leq 2 \max_{\y \in \mathcal{N}} \; \langle \x, \y \rangle$.
\end{claim}

\section{Largest Eigenvalue of Random Kernel Matrices}
In this section, we provide the upper bound on the largest eigenvalue of a random kernel matrix, constructed using polynomial or Gaussian kernels. Notice that the entries of a random kernel matrix are dependent. For example any triplet of entries $(i,j)$, $(j,k)$ and $(k,i)$ are mutually dependent. Additionally, we deal with vectors drawn from general subgaussian distributions, and therefore, the coordinates within a random vector need not be independent. 
%All omitted proofs for this section can be found in Appendix~\ref{app:proof}.

We start off with a simple lemma, to bound the Euclidean norm of a  subgaussian random vector. A random vector $\x$ is {\em centered} if $\E[\x] = 0$.
\begin{lemma} \label{lem:normx}
Let $\x_1 \etc \x_n \in \R^d$ be independent centered subgaussian vectors. Then for all $i \in [n]$,  $\Pr [ \| \x_i \| \geq C \sqrt{d}  ] \leq  \exp(-C'd)$ for constants $C,C'$. 
\end{lemma}
\begin{proof}
To this end, note that since $\x_i$ is a subgaussian vector (from Definition~\ref{def:subgauss})
$$ \Pr \left [ | \langle \x_i, \y \rangle | \geq C \sqrt{d}/2 \right ] \leq 2 \exp(-C_2 d),$$
for constants $C$ and $C_2$, any unit vector $\y \in S^{d-1}$.  Taking the union bound over a $(1/2)$-net ($\mathcal{N}$) in $S^{d-1}$, and using Proposition~\ref{prop:nets} for the size of the nets (which is at most $5^d$ as $\epsilon=1/2$), we get that
$$ \Pr \left [ \max_{\y \in \mathcal{N}} | \langle \x_i, \y \rangle | \geq C \sqrt{d}/2 \right ] \leq  \exp(-C_3 d),$$
From Claim~\ref{claim:net}, we know that $ \| \x_i \| \leq 2 \max_{\y \in \mathcal{N}} \; \langle \x_i, \y \rangle$. Hence, $ \Pr \left [ \| \x_i \| \geq C \sqrt{d}  \right ] \leq  \exp(-C'  d)$.
\end{proof}

\paragraph{Polynomial Kernel.} We now establish the bound on the spectral norm of a polynomial kernel random matrix. We assume $\x_1,\dots,\x_n$ are independent vectors drawn according to a centered subgaussian distribution over $\R^d$. Let $K_p$ denote the kernel matrix obtained using $\x_1,\dots,\x_n$ in a polynomial kernel. Our idea to split the kernel matrix $K_p$ into its diagonal and off-diagonal parts, and then bound the spectral norms of these two matrices separately.  The diagonal part contains independent entries of the form $(a\|\x_i\|^2+b)^p$, and we use Lemma~\ref{lem:normx} to bound its spectral norm. Dealing with the off-diagonal part of $K_p$ is trickier because of the dependence between the entries, and here we bound the spectral norm by its Frobenius norm. We also verify the upper bounds provided in the following theorem by conducting numerical experiments (see Figure~\ref{fig1}).
\begin{theorem} \label{thm:poly}
Let $\x_1 \etc \x_n \in \R^d$ be  independent centered subgaussian vectors. Let $p \in \N$, and let $K_p$ be the $n \times n$ matrix with $(i,j)$th entry $K_{p_{ij}}=(a\langle \x_i,\x_j \rangle + b)^p$. Assume that $n \le \exp(C_1 d)$ for a constant $C_1$. Then there exists constants $C_0, C_0'$ such that
\[ \Pr \left[ \norm{K_p}  \ge  C_0^p |a|^p d^p n + 2^{p+1} |b|^p n \right ] \le \exp(-C_0' d).\]
\end{theorem}
\begin{proof}
To prove the theorem, we split the kernel matrix $K_p$ into the diagonal and off-diagonal parts. Let $K_p = D + W$, where $D$ represents the diagonal part of $K_p$ and $W$ the off-diagonal part of $K_p$. Note that
\begin{align*}  \| K_p \| \leq  \| D \| + \| W \| \leq \| D \| + \| W \|_F. \end{align*}
Let us estimate the norm of the diagonal part $D$ first.  From Lemma~\ref{lem:normx}, we know that for all $i \in [n]$ with $C_3=C'$,
$$\Pr \left [ \| \x_i \| \geq C \sqrt{d} \right ]  = \Pr \left [ \| \x_i \|^2 \geq (C \sqrt{d})^2  \right ] \leq  \exp(-C_3  d).$$
Instead of $\| \x \|_i^2$, we are interested in bounding $(a \| \x_i \|^2 + b)^p$. 
\begin{align} \label{eqn:first} \Pr \left [ \| \x_i \|^2 \geq (C  \sqrt{d})^2  \right ] = \Pr \left [ (a \| \x_i \|^2 + b)^p \geq (a (C  \sqrt{d})^2 + b)^p \right ]. \end{align}
Consider $(a(C \sqrt{d})^2 +b)^p$. A simple inequality to bound $(a(C \sqrt{d})^2+ b)^p$ is\footnote{For any $a,b,m \in \R$ and $p \in \N$, $(a\cdot m + b)^p \leq 2^p (|a|^p |m|^p + |b|^p)$.}
$$(a(C \sqrt{d})^2+ b)^p \leq 2^p(|a|^p(C \sqrt{d})^{2p} + |b|^p).$$
Therefore, 
$$ \Pr \left [ (a \| \x_i \|^2 + b)^p \geq 2^p(|a|^p(C  \sqrt{d})^{2p} + |b|^p)  \right ] \leq \Pr \left [ (a \| \x_i \|^2 + b)^p \geq (a (C  \sqrt{d})^2 + b)^p \right ].$$
Using~\eqref{eqn:first} and substituting in the above equation, for any $i \in [n]$
$$ \Pr \left [ (a \| \x_i \|^2 + b)^p \geq 2^p (|a|^p  C^{2p} d^p + |b|^p) \right ] \leq \Pr \left [ \| \x_i \| \geq C \sqrt{d} \right ] \leq \exp(-C_3  d).$$
By applying a union bound over all $n$ non-zero entries in $D$, we get that for all $i \in [n]$
$$ \Pr \left [ (a \| \x_i \|^2 + b)^p \geq 2^p(|a|^p C^{2p} d^p  + |b|^p) \right ] \leq n \cdot \exp(-C_3  d) \leq \exp(C_1 d) \cdot \exp(-C_3  d)  \leq \exp(-C_4 d),$$
as we assumed that $n \le \exp(C_1 d)$. This implies that
\begin{align} \label{eqn:five} \Pr [ \| D \| \geq 2^p(|a|^p C^{2p} d^p + |b|^p)]  \leq \exp(-C_4 d).\end{align}

We now bound the spectral norm of the off-diagonal part $W$ using Frobenius norm as an upper bound on the spectral norm. Firstly note that for any $\y \in \R^d$, the random variable $\langle \x_i, \y \rangle$ is subgaussian with its $\psi_2$-norm at most $C_5 \| \y \|$ for some constant $C_5$.  This follows as:
$$\| \langle \x_i, \y \rangle \|_{\psi_2}  := \inf \left \{t > 0 \,:\, \E[\exp( \langle \x_i, \y \rangle ^2/t^2)] \leq 2 \right \} \leq C_5 \| \y \|.$$
Therefore, for a fixed $\x_j$, $\| \langle \x_i,\x_j \rangle \|_{\psi_2} \leq C_5 \| \x_j \|$. 
For $i \neq j$, conditioning on $\x_j$,
$$ \Pr \left [| \langle \x_i, \x_j \rangle | \geq \tau \right ] = \E_{\x_j} \left [ \Pr \left [ | \langle \x_i, \x_j \rangle | \geq \tau \;|\; \x_j \right ] \right ]. $$
From Claim~\ref{claim:norm}, 
$$ \E_{\x_j} \left [ \Pr \left [ | \langle \x_i, \x_j \rangle | \geq \tau \;|\; \x_j \right ] \right ] \leq \E_{\x_j} \left [ \exp \left ( \frac{-C_6 \tau^2}{\| \langle \x_i,\x_j \rangle \|_{\psi_2}^2} \right ) \right ] \leq \E_{\x_j} \left [ \exp \left ( \frac{-C_6 \tau^2}{(C_5 \| \x_j \|)^2} \right ) \right ] = \E_{\x_j} \left [ \exp \left ( \frac{-C_7 \tau^2}{\| \x_j \|^2} \right ) \right ]  , $$
where the last inequality uses the fact that $\| \langle \x_i,\x_j \rangle \|_{\psi_2} \leq C_5 \| \x_j \|$.
Now let us condition the above expectation on the value of $\| \x_j \|$ based on whether $\| \x_j \| \geq C \sqrt{d}$ or $\| \x_j \|  < C \sqrt{d}$. We can rewrite
\begin{multline*} 
\E_{\x_j} \left [ \frac{-C_7 \tau^2}{\| \x_j \|^2} \right ]  
 \leq \E_{\x_j} \left [ \exp \left ( \frac{-C_7 \tau^2}{C^2 d} \right) \; \Bigg|\; \| \x_j \| < C \sqrt{d} \right ] \Pr [ \| \x_j \| < C \sqrt{d} ] \\ + \E_{\x_j} \left [ \exp \left ( \frac{-C_7 \tau^2}{\| \x_j \|^2} \right ) \:\Bigg|\: \| \x_j \| \geq C \sqrt{d} \right ] \Pr [ \| \x_j \| \geq C \sqrt{d} ].
\end{multline*}
The above equation can be easily be simplified as:
\begin{align*}
\E_{\x_j} \left [ \frac{-C_7 \tau^2}{\| \x_j \|^2} \right ]  & \leq \exp \left ( \frac{-C_8 \tau^2}{d} \right ) + \E_{\x_j} \left [ \exp \left ( \frac{-C_7 \tau^2}{\| \x_j \|^2} \right ) \:\Bigg|\: \| \x_j \| \geq C \sqrt{d} \right ] \Pr [ \| \x_j \| \geq C \sqrt{d} ].
\end{align*}
From Lemma~\ref{lem:normx}, $  \Pr [ \| \x_j \| \geq C \sqrt{d} ] \leq  \exp(-C_3 d)$, and 
$$ \E_{\x_j} \left [ \exp \left ( \frac{-C_7 \tau^2}{\| \x_j \|^2} \right ) \:\Bigg|\: \| \x_j \| \geq C \sqrt{d} \right ] \leq 1.$$
This implies that as $\Pr [ \| \x_j \| \geq C \sqrt{d} ] \leq   \exp(-C_3 d)$),
$$  \E_{\x_j} \left [ \exp \left ( \frac{-C_7 \tau^2}{\| \x_j \|^2} \right ) \:\Bigg|\: \| \x_j \| \geq C \sqrt{d} \right ]  \Pr [ \| \x_j \| \geq C \sqrt{d} ] \leq   \exp(-C_3 d).$$
Putting the above arguments together, 
\begin{align*}
 \Pr \left [| \langle \x_i, \x_j \rangle | \geq \tau \right ]  & = \E_{\x_j} \left [ \Pr \left [ | \langle \x_i, \x_j \rangle | \geq \tau \;|\; \x_j \right ] \right ]  \leq  \exp \left ( \frac{-C_8 \tau^2}{d} \right ) +  \exp(-C_3 d).
\end{align*}
Taking a union bound over all $(n^2-n) < n^2$ non-zero entries in $W$,
$$ \Pr \left [\max_{i \neq j} | \langle \x_i, \x_j \rangle | \geq \tau \right ]  \leq n^2 \left ( \exp \left ( \frac{-C_8 \tau^2}{d} \right ) +  \exp(-C_3 d) \right ).$$
Setting $\tau = C\cdot d$ in the above and using the fact that $n \leq \exp(C_1 d)$, 
\begin{align}
\label{eqn:third} \Pr \left [\max_{i \neq j} | \langle \x_i, \x_j \rangle | \geq C \cdot d \right ] \leq \exp(-C_9 d). \end{align}
We are now ready to bound the Frobenius norm of $W$.
\begin{align*}
\|W\|_F = \left (\sum_{i \neq j} (a \langle \x_i,\x_j \rangle + b)^{2p} \right )^{1/2} \leq  \left ( n^2 2^{2p} \left (|a|^{2p} \langle \x_i,\x_j \rangle^{2p} + |b|^{2p}\right  ) \right)^{1/2} \leq  n 2^p \left (|a|^p | \langle \x_i,\x_j \rangle |^p + |b|^p \right  ). 
\end{align*}
Plugging in the probabilistic bound on $| \langle \x_i, \x_j \rangle |$ from~\eqref{eqn:third} gives,
\begin{align} \label{eqn:six} 
\Pr \left [ \| W \|_F \geq  n 2^p \left (|a|^p |C^p d^p + |b|^p \right  ) \right ] & \leq \Pr \left [ n 2^p \left (|a|^p | \langle \x_i,\x_j \rangle |^p + |b|^p \right  ) \geq   n 2^p \left (|a|^p |C^p d^p + |b|^p \right  ) \right ] \nonumber \\  
& \leq \exp(-C_9 d).
\end{align}
Plugging bounds on $\| D \|$ (from~\eqref{eqn:five}) and $\| W \|_F$ (from~\eqref{eqn:six}) to upper bound $\| K_p \| \leq \|D\| +  \|W\|_F$ yields that there exists constants $C_0$ and $C_0'$ such that,
$$ \Pr \left [ \| K_p \| \geq C_0^p |a|^p d^p n + 2^{p+1} |b|^p n \right ] \leq \Pr \left [  \|D\| +  \|W\|_F \geq C_0^p |a|^p d^p n + 2^{p+1} |b|^p n  \right ] \leq \exp(-C_0' d).$$
This completes the proof of the theorem. The chain of constants can easily be estimated starting with the constant in the definition of the subgaussian random variable.
\end{proof}

\begin{remark}
Note that for our proofs it is only necessary that $\x_1,\dots,\x_n$ are independent random vectors, but they need not be identically distributed.  This spectral norm upper bound on $K_p$ (again with exponentially high probability) could be improved to $$O\left (C_0^p |a|^p(d^p +d^{p/2} n) +2^{p +1}n |b|^p \right),$$ with a slightly more involved analysis (omitted in this extended abstract). For an even $p$, the expectation of every individual entry of the matrix $K_p$ is positive, which provides tight examples for this bound.
\end{remark}

\paragraph{Gaussian Kernel.}  We now establish the bound on the spectral norm of a Gaussian kernel random matrix. Again assume $\x_1,\dots,\x_n$ are independent vectors drawn according to a centered subgaussian distribution over $\R^d$. Let $K_g$ denote the kernel matrix obtained using $\x_1,\dots,\x_n$ in a Gaussian kernel. Here an upper bound of $n$ on the spectral norm on the kernel matrix follows trivially as all entries of $K_g$ are less than equal to $1$.  We show that this bound is tight, in that for small values of $a$, with high probability the spectral norm is at least $\Omega(n)$. 

In fact, it is impossible to obtain better than $O(n)$ upper bound on the spectral norm of $K_g$ without additional assumptions on the subgaussian distribution, as illustrated by this example: Consider a distribution over $\R^d$, such that a random vector drawn from this distribution is a zero vector $(0)^d$ with probability $1/2$ and uniformly distributed over the sphere in $\R^d$ of radius $2\sqrt{d}$ with probability $1/2$. A random vector $\x$ drawn from this distribution is isotropic and subgaussian, but $\Pr[\x=(0)^d]=1/2$. Therefore, in $\x_1,\dots,\x_n$ drawn from this distribution, with high probability more than a constant fraction of the vectors will be $(0)^d$. This means that a proportional number of entries of the matrix $K_g$ will be $1$, and the norm will be $O(n)$ regardless of $a$. 

This situation changes, however, when we add the additional assumption that $\x_1,\dots,\x_n$ have independent centered subgaussian coordinates\footnote{Some of the commonly used subgaussian random vectors such as the standard normal, Bernoulli satisfy this additional assumption.}  (i.e., each $\x_i$ is drawn from a product distribution formed from some $d$ centered univariate subgaussian distributions). In that case, the kernel matrix $K_g$ is a small perturbation of the identity matrix, and we show that the spectral norm of $K_g$ is with high probability bounded by an absolute constant (for $a = \Omega(\log n/d)$).  For this proof, similar to Theorem~\ref{thm:poly}, we split the kernel matrix into its diagonal and off-diagonal parts. The spectral norm of the off-diagonal part is again bounded by its Frobenius norm. We also verify the upper bounds presented in the following theorem by conducting numerical experiments (see Figure~\ref{fig2}).

\begin{theorem} \label{thm:rbf}
Let $\x_1 \etc \x_n \in \R^d$ be independent centered subgaussian vectors. Let $a > 0$, and let $K_g$ be the $n \times n$ matrix with $(i,j)$th entry $K_{g_{ij}} = \exp( -a \| \x_i - \x_j \|^2)$. Then there exists constants $c,c_0,c_0',c_1$ such that
\renewcommand{\theenumi}{\alph{enumi}}
\begin{CompactEnumerate}
\item \label{part:1} $\| K_g \| \leq n$.
\item \label{part:2} If $a < c_1/d$, $\Pr \left [ \| K_g \| \geq c_0 n \right] \geq 1 - \exp(-c_0' n)$.
\item \label{part:3} If all the vectors $\x_1,\dots,\x_n$ satisfy the additional assumption of having independent centered subgaussian coordinates, and assume $n \le \exp(C_1 d)$ for a constant $C_1$. Then for any $\delta > 0$ and $a  \geq (2+\delta)\frac{\log n}{d}$, $\Pr \left [ \| K_g \| \geq 2 \right] \leq \exp(-c \zeta^2 d)$ with $\zeta > 0$ depending only on $\delta$.
\end{CompactEnumerate}
\end{theorem}
\begin{proof}
Proof of Part~\ref{part:1}) is straightforward as all entries of $K_g$  do not exceed $1$. 

Let us prove the lower estimate for the norm in Part~\ref{part:2}). For $i=1,\dots,n$ define
$$ Z_i = \sum_{j=\frac n 2+1}^n K_{g_{ij}}.$$
From Lemma~\ref{lem:normx} for all $i \in [n]$, $\Pr \left [ \| \x_i \| \geq C \sqrt{d}  \right ] \leq  \exp(-C'd)$. In other words, $\| \x_i\|$ is less than $C \sqrt{d}$ for all $i \in [d]$ with probability at least $1-\exp(-C'd)$. Let us call this event $\mathcal{E}_1$. Under $\mathcal{E}_1$ and assumption $a < c_1/d$, $\E[Z_i] \geq c_2 n$ and $\E[Z_i^2] \leq c_3 n^2$. Therefore, by Paley-Zygmund inequality (under event $\mathcal{E}_1$),
\begin{align}\label{eqn:z} \Pr [Z_i \geq c_4 n] \geq c_5.\end{align}
Now $Z_1,\dots,Z_n$ are not independent random variables. But if we condition on $\x_{n/2+1}, \ldots, \x_n$, then $Z_1,\dots,Z_{n/2}$ become independent (for simplicity, assume that $n$ is divisible by $2$). Thereafter, an application of Chernoff bound on $Z_1,\dots,Z_{n/2}$ using the probability bound from~\eqref{eqn:z} (under conditioning on $\x_{n/2+1}, \ldots, \x_n$ and event $\mathcal{E}_1$) gives:
$$ \Pr \left [Z_i \geq c_4 n \mbox{ for at least } c_5 n \mbox{ entries } Z_i \in  \{Z_1,\dots,Z_{n/2}\} \right ] \geq 1 - \exp(-c_6 n).$$
The first conditioning can be removed by taking the expectation with respect to  $\x_{n/2+1}, \ldots, \x_n$ without disturbing the exponential probability bound.  Similarly, conditioning on event $\mathcal{E}_1$ can also be easily removed.

Let $K_g'$ be the submatrix of $K_g$ consisting of rows $1 \leq i \leq n/2$ and columns $n/2+1 \leq j \leq n$. Note that $\| K_g' \| \geq \u^\top K_g' \u$, where $\u = \left ( \sqrt{\frac 2 n},\dots,\sqrt{\frac 2 n} \right )$ (of dimension $n/2$). Then
\begin{align*}
\Pr[ \| K_g \| \leq c_0 n] & \leq \Pr [ \| K_g' \| \leq c_7 n] \leq \Pr [ \u^\top K_g' \u \leq c_7 n] \\
& \Pr \left [ \frac{2}{n} \sum_{i=1}^{n/2} Z_i \leq c_7 n \right ] \leq \exp(-c_0' n).
\end{align*}
The last line follows as from above arguments with exponentially high probability above more than $\Omega(n)$ entries in $Z_1,\dots,Z_{n/2}$ are greater than $\Omega(n)$, and by readjusting the constants.

Proof of Part~\ref{part:3}): As in Theorem~\ref{thm:poly}, we split the matrix $K_g$ into the diagonal ($D$) and the off-diagonal part ($W$) (i.e., $K_g=D+W$). It is simple to observe that $D=\mathbb{I}_n$, therefore we just concentrate on $W$. The $(i,j)$th entry in $W$ is $\exp( -a \| \x_i - \x_j \|^2)$, where $\x_i$ and $\x_j$ are independent vectors with independent centered subgaussian coordinates. Therefore, we can use Hoeffding's inequality, for fixed $i,j$,
\begin{align} \label{eqn:hoeff}
 \Pr \left[ \exp( -a \| \x_i - \x_j \|^2) \geq \exp(-a(1-\zeta)d) \right ] = \Pr \left [ \frac{\| \x_i - \x_j \|^2}{d} \leq (1-\zeta) \right ] \leq \exp(-c_8 \zeta^2 d),
 \end{align}
where we used the fact that if a random variable is subgaussian then its square is a subexponential random variable~\cite{V11}.\footnote{We call a random variable $x \in \R$ subexponential if there exists  a constant $C > 0$ if $\Pr[ |x| > t]  \leq 2 \exp(-t/C)$ for all $t > 0$. } To estimate the norm of $W$, we bound it by its Frobenius norm. If $a \geq (2+\delta)\frac{\log n}{d}$, then we can choose $\zeta > 0$ depending on $\delta$ such that $n^2 \exp(-a (1-\zeta)d) \leq 1$.  Hence,
\begin{align*}
 \Pr [ \| K_g \| \geq 2 ]  & \leq \Pr [ \| D \| + \| W \|_F \geq 2] = \Pr [ \|W\|_F \geq 1] \\
& = \Pr \left [ \sum_{1 \leq i,j \leq n, i \neq j}  \exp( -a \| \x_i - \x_j \|^2) \geq  1 \right ] \\
&  \leq \Pr \left [ \sum_{1 \leq i,j \leq n, i \neq j}  \exp( -a \| \x_i - \x_j \|^2) \geq  n^2 \exp(-a(1-\zeta)d) \right ] \\
& \leq \Pr \left [ \sum_{1 \leq i,j \leq n}  \exp( -a \| \x_i - \x_j \|^2) \geq  n^2 \exp(-a(1-\zeta)d) \right ] \\
& \leq n^2  \Pr \left[ \max_{1 \leq i,j \leq n} \exp( -a \| \x_i - \x_j \|^2) \geq \exp(-a(1-\zeta)d) \right ] \\
& \leq n^2  \exp(-c_8 \zeta^2 d) \\
&  \leq \exp(-c \zeta^2 d) \mbox{ for some constant } c.
\end{align*}
The first equality follows as $\| D \| =1$, and the second-last inequality follows from~\eqref{eqn:hoeff}. This completes the proof of the theorem. Again the long chain of constants can easily be estimated starting with the constant  in the definition of the subgaussian random variable.
\end{proof}

\begin{remark}
Note that again the $\x_i$'s need not be identically distributed.  Also as mentioned earlier, the analysis in Theorem~\ref{thm:rbf} could easily be extended to other exponential kernels such as the Laplacian kernel.
\end{remark}

\begin{figure*}[htb]
\centering
\subfigure[Polynomial Kernel]{\includegraphics[width=3.2in]{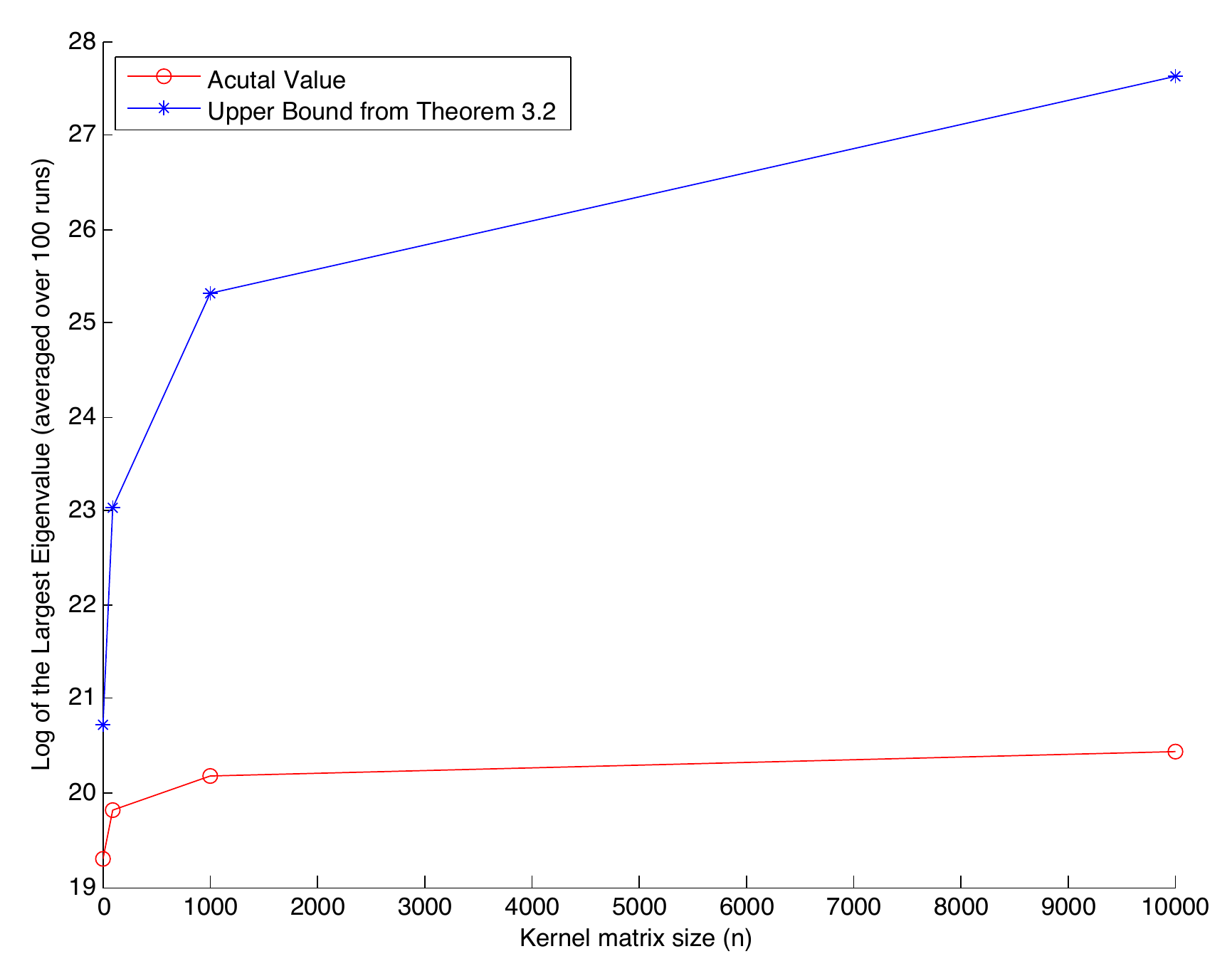} \label{fig1}}
\subfigure[Gaussian Kernel]{\includegraphics[width=3.2in]{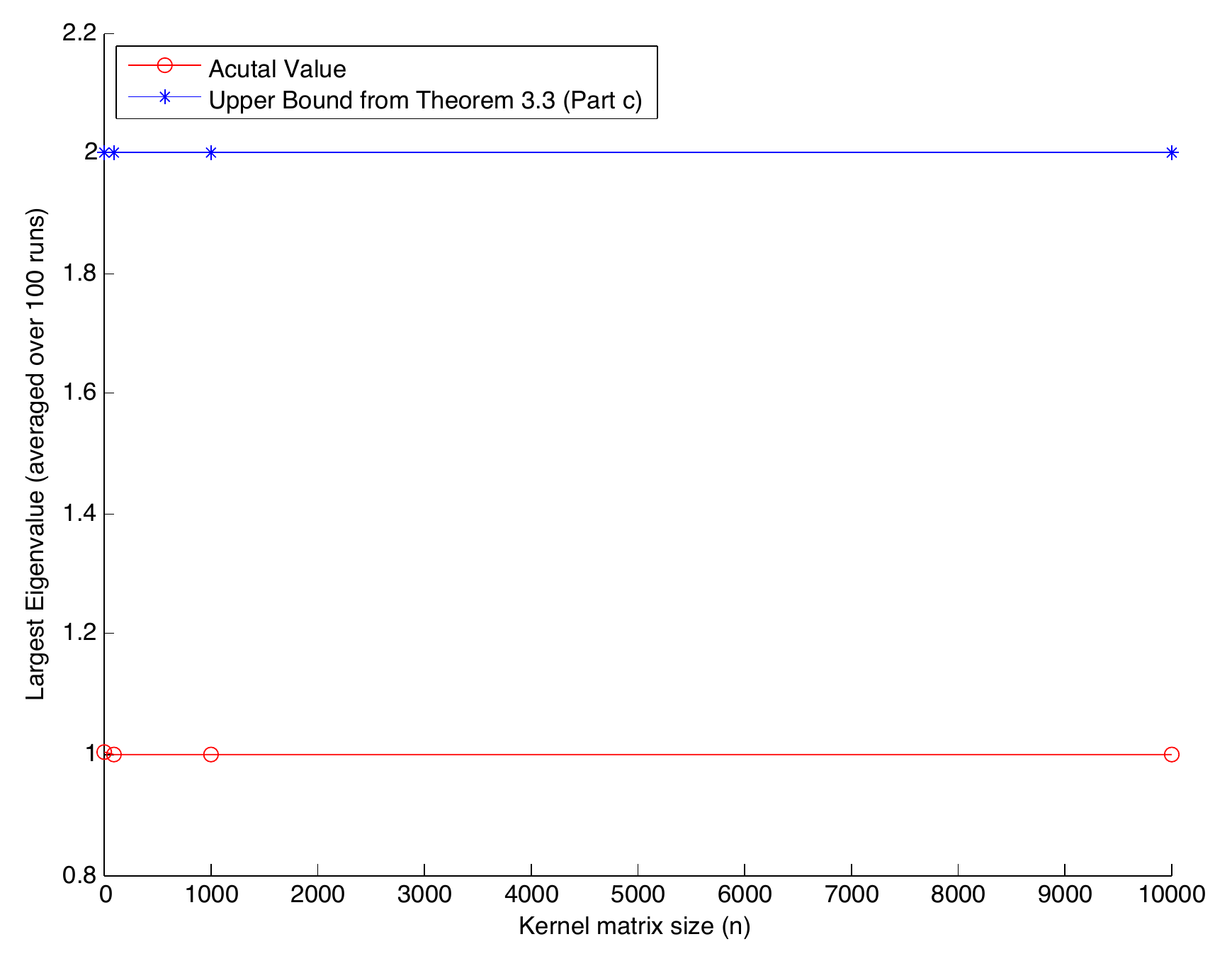} \label{fig2}}
\caption{Largest eigenvalue distribution for random kernel matrices constructed with a polynomial kernel (left plot) and a Gaussian kernel (right plot). The actual value plots are constructed by averaging over 100 runs, and in each run we draw $n$ independent standard Gaussian vectors in $d=100$ dimensions. The predicted values are computed from bounds in Theorems~\ref{thm:poly} and~\ref{thm:rbf} (Part~\ref{part:3}). The kernel matrix size $n$ is varied from $10$ to $10000$ in multiples of $10$. For the polynomial kernel, we set $a=1,b=1$, and $p=4$, and for the Gaussian kernel $a=3\log (n)/d$. Note that our upper bounds are fairly close to the actual results. For the Gaussian kernel, the actual values are very close to $1$. }
\label{datastream}
\end{figure*}

\section{Application: Privately Releasing Kernel Ridge Regression Coefficients}\label{sec:private}
We consider an application of Theorems~\ref{thm:poly} and~\ref{thm:rbf} to obtain noise lower bounds for privately releasing coefficients of kernel ridge regression. For privacy violation, we consider a generalization of \emph{blatant non-privacy}~\cite{DiNi03} referred to as attribute non-privacy (formalized in~\cite{KRSU10}). Consider a database $D \in \R^{n \times d+1}$ that contains, for each individual $i$, a sensitive attribute $y_i\in\{0,1\}$ as well as some other information $\x_i \in \R^d$ which is assumed to be known to the attacker. The $i$th record is thus $(\x_i,y_i)$. Let $X \in \R^{n\times d}$ be a matrix whose $i$th row is $\x_i$, and let $\y = (y_1,\dots,y_n)$. We denote the entire database $D=(X|\y)$ where $|$ represents vertical concatenation. Given some released information $\rho$, the attacker constructs an estimate $\hat\y$ that she hopes is close to $\y$.  We measure the attack's success in terms of the Hamming distance $d_H(\y,\hat\y)$. A scheme is \emph{not} attribute private if an attacker can consistently get an estimate that is within distance $o(n)$. Formally:

\begin{definition}[Failure of Attribute Privacy~\cite{KRSU10}] \label{def:apy} 
A (randomized) mechanism $\mathcal{M} \,:\, \R^{n \times d+1} \to \R^l$ is said to allow $(\theta,\gamma)$ attribute reconstruction if there exists a setting of the nonsensitive attributes $X \in \R^{n \times d}$ and an algorithm (adversary) $\AAA \::\; \R^{n \times d} \times \R^l \to \R^n$ such that for every $\y \in \{0,1\}^n$,
\begin{align*}
\Pr_{\rho \gets \mathcal{M}((X|\y))} [\AAA(X,\rho) = \hat{\y} \,: \, d_H(\y,\hat{\y}) \leq \theta  ] \geq 1 - \gamma.
\end{align*}
\end{definition}
Asymptotically, we say that a mechanism is \emph{attribute nonprivate} if there is an infinite sequence of $n$ for which $\mathcal M$ allows $(o(n), o(1))$-reconstruction. Here $d=d(n)$ is a function of $n$.  We say the attack $\AAA$ is \emph{efficient} if it runs in time $\mbox{poly}(n,d)$. 

\paragraph{Kernel Ridge Regression Background.} One of the most basic regression formulation is that of ridge regression~\cite{hoerl1970ridge}. Suppose that we are given a dataset $\{(\x_i,y_i)\}_{i=1}^n$ consisting of $n$ points with $\x_i \in \R^d$ and $y_i \in \R$. Here $\x_i$'s are referred to as the {\em regressors} and $y_i$'s are the {\em response variables}. In linear regression the task is to find a linear function that models the dependencies between $\x_i$'s and the $y_i$'s. A common way to prevent overfitting in linear regression is by adding a penalty regularization term (also known as {\em shrinkage} in statistics). In kernel ridge regression~\cite{saunders1998ridge}, we assume a model of form $y = f(\x) + \xi$, where we are trying to estimate the regression function $f$ and $\xi$ is some unknown vector that accounts for discrepancy between the actual response ($y$) and predicted outcome ($f(\x)$). Given a reproducing kernel Hilbert space $\HHH$ with kernel $\kappa$, the goal of ridge regression kernel ridge regression is to estimate the unknown function $f^\star$ such the least-squares loss defined over the dataset with a weighted penalty based on the squared Hilbert norm is minimized.
\begin{align} \label{eqn:krr}
\mbox{Kernel Ridge Regression:}\;\;\;\;\;\;\;\;\;\; \mbox{argmin}_{f \in \HHH} \; \left ( \frac{1}{n} \sum_{i=1}^n (y_i - f(\x_i))^2 + \lambda \| f \|_{\HHH}^2 \right ),
\end{align}
where $\lambda > 0$ is a regularization parameter. By representer theorem~\cite{scholkopf2001generalized}, any solution $f^\star$ for~\eqref{eqn:krr}, takes the form 
\begin{align} \label{eqn:fast}
f^\star(\cdot) = \sum_{i=1}^n \alpha_i \kappa(\cdot,\x_i),
\end{align}
where $\alpha=(\alpha_1,\dots,\alpha_n)$ is known as the kernel ridge regression coefficient vector. Plugging this representation into~\eqref{eqn:krr} and solving the resulting optimization problem (in terms of $\alpha$ now), we get that the minimum value is achieved for $\alpha=\alpha^\star$, where
\begin{align} \label{eqn:alpha}
\alpha^\star = (K + \lambda \mathbb{I}_n)^{-1} \y, \mbox{ where } K \mbox{ is the kernel matrix with $K_{ij}=\kappa(\x_i,\x_j)$ and } \y=(y_1,\dots,y_n).
\end{align}
Plugging this $\alpha^\star$ from~\eqref{eqn:alpha} in to~\eqref{eqn:fast}, gives the final form for estimate $f^\star(\cdot)$. This means that for a new point $\x \in \R^d$, the predicted response is $f^{\star}(\x) = \sum_{i=1}^n \alpha_i^\star \kappa(\x,\x_i)$ where $\alpha^\star = (K + \lambda \mathbb{I}_n)^{-1} \y$ and $\alpha^\star=(\alpha^\star_1,\dots,\alpha^\star_n)$. Therefore, knowledge of $\alpha^\star$ and $\x_1,\dots,\x_n$ suffices for using the regression model for making future predictions. 

If $K$ is constructed using a polynomial kernel (defined in~\resref{poly}) then the above procedure is referred to as the {\em polynomial kernel ridge regression}, and similarly if $K$ is constructed using a Gaussian kernel (defined in~\resref{rbf}) then the above procedure is referred to as the {\em Gaussian kernel ridge regression}.

\paragraph{Reconstruction Attack from Noisy $\alpha^\ast$.} Algorithm~\ref{alg:krr} outlines the attack. The privacy mechanism releases a noisy approximation to $\alpha^\star$. Let $\tilde{\alpha}$ be this noisy approximation, i.e., $\tilde{\alpha} = \alpha^\star + \e$ where $\e$ is some unknown noise vector. The adversary tries to reconstruct an approximation $\hat{\y}$  of $\y$ from $\tilde{\alpha}$. The adversary solves the following $\ell_2$-minimization problem to construct~$\hat{\y}$:
\begin{align} \label{eqn:l2} \mbox{min}_{\z \in \R^n} \| \tilde{\alpha} - (K + \lambda \mathbb{I}_n)^{-1} \z \| .\end{align}
In the setting of attribute privacy, the database $D=(X|\y)$. Let $\x_1,\dots,\x_n$ be the rows of $X$, using which the adversary can construct $K$ to carry out the attack. Since  the matrix $K + \lambda \mathbb{I}_n$ is invertible for $\lambda > 0$ as $K$ is a positive semidefinite matrix, the solution to~\eqref{eqn:l2} is simply $\z=(K + \lambda \mathbb{I}_n) \tilde{\alpha}$, element-wise rounding of which to closest $0,1$ gives $\hat{\y}$.
\begin{algorithm}[h]
\caption{Reconstruction Attack from Noisy Kernel Ridge Regression Coefficients}
\label{alg:krr}
\mbox{\textbf{Input:} Public information $X \in \R^{n \times d}$, regularization parameter $\lambda$, and $\tilde{\alpha}$ (noisy version of $\alpha^\star$ defined in~\eqref{eqn:alpha}).}
\begin{algorithmic}[1]
\STATE Let $\x_1,\dots,\x_n$ be the rows of $X$, construct the kernel matrix $K$ with $K_{ij}=\kappa(\x_i,\x_j)$
\STATE {\textbf Return} $\hat{\y} = (\hat{y}_1,\dots,\hat{y}_n)$ defined as follows:
\begin{align*}
\hat{y}_i = \left\{ \begin{array}{rl}
0 &\mbox{ if $i$th entry in $(K + \lambda \mathbb{I}_n) \tilde{\alpha}< 1/2$} \\
1 &\mbox{ otherwise}
\end{array} \right.
\end{align*}
\end{algorithmic}
\end{algorithm}

\begin{lemma} \label{lem:error}
Let $\tilde{\alpha} = \alpha^\star + \e$, where $\e \in \R^n$ is some unknown (noise) vector. If $\| \e \|_{\infty} \leq \beta$ (absolute value of all entries in $\e$ is less than $\beta$),  then $\hat{\y}$ returned by Algorithm~\ref{alg:krr} satisfies, $d_H(\y,\hat{\y}) \leq 4 (K+\lambda)^2\beta^2 n$. In particular, if $\beta = o \left (\frac{1}{\| K \| + \lambda} \right )$, then $d_H(\y,\hat{\y}) = o(n)$.
\end{lemma}
\begin{proof}
Since $\alpha^\star = (K + \lambda \mathbb{I}_n)^{-1} \y$, $\tilde{\alpha} = (K + \lambda \mathbb{I}_n)^{-1} \y + \e$. Now multiplying $(K + \lambda \mathbb{I}_n)$ on both sides gives,
$$ (K + \lambda \mathbb{I}_n) \tilde{\alpha} = \y + (K + \lambda \mathbb{I}_n) \e. $$
Concentrate on $\|  (K + \lambda \mathbb{I}_n) \e \|$. This can be bound as
$$\|  (K + \lambda \mathbb{I}_n) \e \| \le \|  (K + \lambda \mathbb{I}_n) \| \| \e \| = (\| K \| + \lambda) \| \e \|.$$
If the absolute value of all the entries in $\e$ are less than $\beta$ then $\| \e \| \leq \beta \sqrt{n}$. A simple manipulation then shows that if the above hold then $(K + \lambda \mathbb{I}_n) \e$ cannot have more than $4 ( \| K\|+\lambda)^2\beta^2 n$ entries with absolute value above $1/2$. Since $\hat{\y}$ and $\y$ only differ in those entries where $(K + \lambda \mathbb{I}_n) \e$ is greater than $1/2$, it follows that $d_H(\y,\hat{\y}) \leq 4 (\| K \|+\lambda)^2\beta^2 n$. Setting $\beta = o(\frac{1}{\| K \| + \lambda})$ implies $d_H(\y,\hat{\y}) = o(n)$.
\end{proof}

For a privacy mechanism to be attribute non-private, the adversary has to be able reconstruct an $1-o(1)$ fraction of $\y$ with high probability. Using the above lemma, and the different bounds on $\| K \|$ established in Theorems~\ref{thm:poly} and~\ref{thm:rbf}, we get the following lower bounds for privately releasing kernel ridge regression coefficients.
\renewcommand{\theenumi}{\arabic{enumi}}
\begin{proposition} \label{prop:main}
\begin{CompactEn}
\item \label{parta} Any privacy mechanism which for every database $D=(X|\y)$ where $X \in \R^{n \times d}$ and $\y \in \{0,1\}^n$ releases the coefficient vector of a polynomial kennel ridge regression model (for constants $a, b,$ and $p$)  fitted between $X$ (matrix of regressor values) and $\y$ (response vector), by adding $o(\frac{1}{d^p n+ \lambda})$ noise to each coordinate is attribute non-private. The attack that achieves this attribute privacy violation operates in $O(dn^2)$ time.
\item \label{partb} Any privacy mechanism which for every database $D=(X|\y)$ where $X \in \R^{n \times d}$ and $\y \in \{0,1\}^n$ releases the coefficient vector of a Gaussian kennel ridge regression model (for constant $a$) fitted between $X$ (matrix of regressor values) and $\y$ (response vector), by adding $o(\frac{1}{2 + \lambda})$ noise to each coordinate is attribute non-private. The attack that achieves this attribute privacy violation operates in $O(d n^2)$ time.
\end{CompactEn}
\end{proposition}
\begin{proof}
For Part~\ref{parta}, draw each individual $i$'s non-sensitive attribute vector $\x_i$ independently from any $d$-dimensional subgaussian distribution, and use Lemma~\ref{lem:error} in conjunction with Theorem~\ref{thm:poly}. 

For Part~\ref{partb}, draw each individual $i$'s non-sensitive attribute vector $\x_i$ independently from any product distribution formed from some $d$ centered univariate subgaussian distributions, and use Lemma~\ref{lem:error} in conjunction with Theorem~\ref{thm:rbf} (Part~\ref{part:3}).\!\footnote{Note that it is not critical for $\x_i$'s to be drawn from a product distribution. It is possible to analyze the attack even under a (weaker) assumption that each individual $i$'s non-sensitive attribute vector $\x_i$ is drawn independently from a $d$-dimensional subgaussian distribution, by using Lemma~\ref{lem:error} in conjunction with Theorem~\ref{thm:rbf} (Part~\ref{part:1}).}

The time needed to construct the kernel matrix $K$ is $O(d n^2)$, which dominates the overall computation time.
\end{proof}

We can ask how the above distortion needed for privacy compares to typical entries in $\alpha^\star$. The answer is not simple, but there are natural settings of inputs, where the noise needed for privacy becomes comparable with coordinates of $\alpha^\star$, implying that the privacy comes at a steep price. One such example is if the $\x_i$'s are drawn from the standard normal distribution, $\y=(1)^n$, and all other kernel parameters are constant, then the expected value of the corresponding $\alpha^\star$ coordinates match the noise bounds obtained in Proposition~\ref{prop:main}.

Note that Proposition~\ref{prop:main} makes no assumptions on the dimension $d$ of the data, and holds for all values of $n,d$.  This is different from all other previous lower bounds for attribute privacy~\cite{KRSU10,De11,KRS13}, all of which require $d$ to be comparable to $n$, thereby holding only either when the non-sensitive data (the $\x_i$'s) are very high-dimensional or for very small $n$. Also all the previous lower bound analyses~\cite{KRSU10,De11,KRS13} critically rely on the fact that the individual coordinates of each of the $\x_i$'s are independent\footnote{This may not be a realistic assumption in many practical scenarios. For example, an individual's salary and postal address code are correlated and not independent.}, which is not essential for Proposition~\ref{prop:main}.

%All previous noise lower bounds for preventing attribute non-privacy such as that for releasing $k$-way marginal tables~\cite{KRSU10,De11}, linear/logistic regression parameters~\cite{KRS13}, subclass $M$-estimators~\cite{KRS13}, require $d$ to be very big ($\approx \Omega(n)$).\!\footnote{For releasing $k$-way conjunctions the attribute non-privacy lower bounds  in~\cite{KRSU10} work when $d = \Omega(n^{1/k})$ for constant $k$.} This meant that these previous lower bounds worked only when the data is high-dimensional or for very small $n$. 
%Proposition~\ref{prop:main} also shows that high noise is needed while releasing coefficients of ridge regression with Gaussian kernels to prevent simple reconstruction attacks. Releasing coefficients of ridge regression with polynomial kernels seem to be slightly more tolerant to linear reconstruction attacks.

\paragraph{Note on $\ell_1$-reconstruction Attacks.} A natural alternative to~\eqref{eqn:l2} is to use $\ell_1$-minimization (also known as ``LP decoding''). This gives rise to the following linear program:
\begin{align} \label{eqn:l1} \mbox{min}_{\z \in \R^n} \| \tilde{\alpha} - (K + \lambda \mathbb{I}_n)^{-1} \z \|_1. \end{align}
In the context of privacy, the $\ell_1$-minimization approach was first proposed by Dwork \emph{et al.}\ \cite{DMT07}, and recently reanalyzed in different contexts by~\cite{De11,KRS13}. These results have shown that, for some settings, the $\ell_1$-minimization can handle considerably more complex noise patterns than the $\ell_2$-minimization. However, in our setting, since the solutions for~\eqref{eqn:l1} and~\eqref{eqn:l2} are exactly the same ($\z=(K + \lambda \mathbb{I}_n)\tilde{\alpha}$), there is no inherent advantage of using the $\ell_1$-minimization.

\subsection*{Acknowledgements}
We are grateful for helpful initial discussions with Adam Smith and Ambuj Tewari.


\begin{thebibliography}{10}

\bibitem{bousquet2004advanced}
{\sc Bousquet, O., von Luxburg, U., and R{\"a}tsch, G.}
\newblock {Advanced Lectures on Machine Learning}.
\newblock In {\em ML Summer Schools 2003\/} (2004).

\bibitem{cheng2013spectrum}
{\sc Cheng, X., and Singer, A.}
\newblock {The Spectrum of Random Inner-Product Kernel Matrices}.
\newblock {\em Random Matrices: Theory and Applications 2}, 04 (2013).

\bibitem{Choromanski:2012:PDA:2213556.2213570}
{\sc Choromanski, K., and Malkin, T.}
\newblock {The Power of the Dinur-Nissim Algorithm: Breaking Privacy of
  Statistical and Graph Databases}.
\newblock In {\em PODS\/} (2012), ACM, pp.~65--76.

\bibitem{De11}
{\sc De, A.}
\newblock {Lower Bounds in Differential Privacy}.
\newblock In {\em TCC\/} (2012), pp.~321--338.

\bibitem{DiNi03}
{\sc Dinur, I., and Nissim, K.}
\newblock {Revealing Information while Preserving Privacy.}
\newblock In {\em PODS\/} (2003), ACM, pp.~202--210.

\bibitem{do2013spectrum}
{\sc Do, Y., and Vu, V.}
\newblock {The Spectrum of Random Kernel Matrices: Universality Results for
  Rough and Varying Kernels}.
\newblock {\em Random Matrices: Theory and Applications 2}, 03 (2013).

\bibitem{DMNS06}
{\sc Dwork, C., McSherry, F., Nissim, K., and Smith, A.}
\newblock {Calibrating Noise to Sensitivity in Private Data Analysis}.
\newblock In {\em TCC\/} (2006), vol.~3876 of {\em LNCS}, Springer,
  pp.~265--284.

\bibitem{DMT07}
{\sc Dwork, C., McSherry, F., and Talwar, K.}
\newblock {The Price of Privacy and the Limits of LP Decoding}.
\newblock In {\em STOC\/} (2007), ACM, pp.~85--94.

\bibitem{DY08}
{\sc Dwork, C., and Yekhanin, S.}
\newblock {New Efficient Attacks on Statistical Disclosure Control Mechanisms}.
\newblock In {\em CRYPTO\/} (2008), Springer, pp.~469--480.

\bibitem{hoerl1970ridge}
{\sc Hoerl, A.~E., and Kennard, R.~W.}
\newblock {Ridge Regression: Biased Estimation for Nonorthogonal Problems}.
\newblock {\em Technometrics 12}, 1 (1970), 55--67.

\bibitem{jain2013differentially}
{\sc Jain, P., and Thakurta, A.}
\newblock {Differentially Private Learning with Kernels}.
\newblock In {\em ICML\/} (2013), pp.~118--126.

\bibitem{jia2009accurate}
{\sc Jia, L., and Liao, S.}
\newblock {Accurate Probabilistic Error Bound for Eigenvalues of Kernel
  Matrix}.
\newblock In {\em Advances in Machine Learning}. Springer, 2009, pp.~162--175.

\bibitem{karoui2010spectrum}
{\sc Karoui, N.~E.}
\newblock {The Spectrum of Kernel Random Matrices}.
\newblock {\em The Annals of Statistics\/} (2010), 1--50.

\bibitem{KRS13}
{\sc Kasiviswanathan, S.~P., Rudelson, M., and Smith, A.}
\newblock {The Power of Linear Reconstruction Attacks}.
\newblock In {\em SODA\/} (2013), pp.~1415--1433.

\bibitem{KRSU10}
{\sc Kasiviswanathan, S.~P., Rudelson, M., Smith, A., and Ullman, J.}
\newblock {The Price of Privately Releasing Contingency Tables and the Spectra
  of Random Matrices with Correlated Rows}.
\newblock In {\em STOC\/} (2010), pp.~775--784.

\bibitem{lanckriet2004learning}
{\sc Lanckriet, G.~R., Cristianini, N., Bartlett, P., Ghaoui, L.~E., and
  Jordan, M.~I.}
\newblock {Learning the Kernel Matrix with Semidefinite Programming}.
\newblock {\em The Journal of Machine Learning Research 5\/} (2004), 27--72.

\bibitem{mercer1909functions}
{\sc Mercer, J.}
\newblock {Functions of Positive and Negative Type, and their Connection with
  the Theory of Integral Equations}.
\newblock {\em Philosophical transactions of the royal society of London.
  Series A, containing papers of a mathematical or physical character\/}
  (1909), 415--446.

\bibitem{merener2011polynomial}
{\sc Merener, M.~M.}
\newblock {Polynomial-time Attack on Output Perturbation Sanitizers for
  Real-valued Databases}.
\newblock {\em Journal of Privacy and Confidentiality 2}, 2 (2011), 5.

\bibitem{MuthuN12}
{\sc Muthukrishnan, S., and Nikolov, A.}
\newblock {Optimal Private Halfspace Counting via Discrepancy}.
\newblock In {\em STOC\/} (2012), pp.~1285--1292.

\bibitem{rudelson2014recent}
{\sc Rudelson, M.}
\newblock {Recent Developments in Non-asymptotic Theory of Random Matrices}.
\newblock {\em Modern Aspects of Random Matrix Theory 72\/} (2014), 83.

\bibitem{saunders1998ridge}
{\sc Saunders, C., Gammerman, A., and Vovk, V.}
\newblock {Ridge Regression Learning Algorithm in Dual Variables}.
\newblock In {\em ICML\/} (1998), pp.~515--521.

\bibitem{scholkopf2001generalized}
{\sc Sch{\"o}lkopf, B., Herbrich, R., and Smola, A.~J.}
\newblock {A Generalized Representer Theorem}.
\newblock In {\em COLT\/} (2001), pp.~416--426.

\bibitem{scholkopf2001learning}
{\sc Scholkopf, B., and Smola, A.~J.}
\newblock {\em {Learning with Kernels: Support Vector Machines, Regularization,
  Optimization, and Beyond}}.
\newblock MIT Press, 2001.

\bibitem{shawe2004kernel}
{\sc Shawe-Taylor, J., and Cristianini, N.}
\newblock {\em {Kernel Methods for Pattern Analysis}}.
\newblock Cambridge University Press, 2004.

\bibitem{shawe2005eigenspectrum}
{\sc Shawe-Taylor, J., Williams, C.~K., Cristianini, N., and Kandola, J.}
\newblock {On the Eigenspectrum of the Gram matrix and the Generalization Error
  of Kernel-PCA}.
\newblock {\em Information Theory, IEEE Transactions on 51}, 7 (2005),
  2510--2522.

\bibitem{sindhwani2012scalable}
{\sc Sindhwani, V., Quang, M.~H., and Lozano, A.~C.}
\newblock {Scalable Matrix-valued Kernel Learning for High-dimensional
  Nonlinear Multivariate Regression and Granger Causality}.
\newblock {\em arXiv preprint arXiv:1210.4792\/} (2012).

\bibitem{V11}
{\sc Vershynin, R.}
\newblock {Introduction to the Non-asymptotic Analysis of Random Matrices}.
\newblock {\em arXiv preprint arXiv:1011.3027\/} (2010).

\end{thebibliography}
 \end{document}